\documentclass[conference]{IEEEtran}
\ifCLASSINFOpdf
\else
\fi
\newcommand{\vertiii}[1]{{\left\vert\kern-0.25ex\left\vert\kern-0.25ex\left\vert #1 
    \right\vert\kern-0.25ex\right\vert\kern-0.25ex\right\vert}}
\usepackage{amsfonts}
\usepackage{amsmath,amsthm,amssymb}
\usepackage{enumitem}
\usepackage{algorithmic}
\usepackage{algorithm}
\usepackage{caption}
\usepackage{subcaption}
\usepackage{graphicx}
\usepackage{latexsym}
\usepackage{framed}
\usepackage{color}
\usepackage[table]{xcolor}

\usepackage{url}

\theoremstyle{definition}

\newtheorem{lemma}{Lemma}

\newtheorem{proposition}{Proposition}
\theoremstyle{definition} 
\theoremstyle{definition}

\begin{document}
%
\title{Iteratively Reweighted Least Squares Algorithms for L1-Norm Principal Component Analysis}

\author{\IEEEauthorblockN{Young Woong Park}
\IEEEauthorblockA{Cox School of Business\\ 
Southern Methodist University\\
Dallas, Texas 75225\\
Email: ywpark@smu.edu}
\and
\IEEEauthorblockN{Diego Klabjan}
\IEEEauthorblockA{Department of Industrial Engineering\\
 and Management Sciences\\
Northwestern University\\
Evanston, Illinois 60208\\
Email: d-klabjan@northwestern.edu}}


%


\maketitle

\begin{abstract}
Principal component analysis (PCA) is often used to reduce the dimension of data by selecting a few orthonormal vectors that explain most of the variance structure of the data. L1 PCA uses the L1 norm to measure error, whereas the conventional PCA uses the L2 norm. For the L1 PCA problem minimizing the fitting error of the reconstructed data, we propose an exact reweighted and an approximate algorithm based on iteratively reweighted least squares. We provide convergence analyses, and compare their performance against benchmark algorithms in the literature. The computational experiment shows that the proposed algorithms consistently perform best.
\end{abstract}


%
\IEEEpeerreviewmaketitle

\section{Introduction}
\label{section_introduction}

Principal component analysis (PCA) is a technique to find orthonormal vectors, which are a linear combination of the attributes of the data, that explain the variance structure of the data \cite{Jolliffe:2002}. Since a few orthonormal vectors usually explain most of the variance, PCA is often used to reduce dimension of the data by keeping only a few of the orthonormal vectors. These orthonormal vectors are called \textit{principal components} (PCs). 

For dimensionality reduction, we are given target dimension p, the number of PCs. To measure accuracy, given $p$ principal components, first, the original data is projected into the lower dimension using the PCs. Next, the projected data in the lower dimension is lifted to the original dimension using the PCs. Observe that this procedure causes loss of some information if $p$ is smaller than the dimension of the original attribute space. The reconstruction error is defined by the difference between the projected-and-lifted data and the original data. To select the best $p$ PCs, the following two objective functions are usually used:
\begin{center}
\begin{tabular}{l}
$[\textsf{P1}]$ minimization of the reconstruction error,\\
$[\textsf{P2}]$ maximization of the variance of the projected data.\\
\end{tabular}
\end{center}

The traditional measure to capture the errors and variances in \textsf{P1} and \textsf{P2} is the $L_2$ norm. For each observation, we have the vector of the reconstruction error and variance for \textsf{P1} and \textsf{P2}, respectively. Then, the squared $L_2$ norm of the vectors are added over all observations to define the total reconstruction error and variance for \textsf{P1} and \textsf{P2}, respectively. In fact, in terms of a matrix norm, we optimize the Frobenius norm of the reconstruction error and projected data matrices for \textsf{P1} and \textsf{P2}, respectively. With the $L_2$ norm as the objective function, \textsf{P1} and \textsf{P2} are actually equivalent. Further, \textsf{P2} can be efficiently solved by singular value decomposition (SVD) of the data matrix or the eigenvalue decomposition (EVD) of the covariance matrix of the data. However, the $L_2$ norm is usually sensitive to outliers. As an alternative, PCA based on $L_1$ norm has been used to find more robust PCs. 

For \textsf{P1}, instead of the $L_2$ norm, we minimize the sum of the $L_1$ norm of the reconstruction error vectors over all observations. A few heuristics have been proposed for this minimization problem. The heuristic proposed in \cite{Baccini-etal:96} is based on a canonical correlation analysis. The iterative algorithm in \cite{Ke-Kanade:05} assumes that the projected-and-lifted data is a product of two matrices and is then iteratively optimizing by fixing one of the two matrices. The algorithm in \cite{Brooks-etal:12} sequentially reduces the dimension by one. The algorithm is based on the observation that the projection from $k$ to the best fit $k-1$ dimension can be found by solving several linear program (LP) problems for least absolute deviation regression. The algorithms in \cite{Ke-Kanade:05} and \cite{Brooks-etal:12} actually try to find the best fitting subspace, where in the objective function the original data is approximated by the multiplication of two matrices, PC and score matrices. This approximation is not the same as the reconstructed matrix by PCs, while the ultimate goal is still minimizing the reconstruction errors. 

The $L_1$ norm for \textsf{P2} has also been studied. This problem is often called the \textit{projection pursuit $L_1$-PCA}. In this context, we maximize the sum of the $L_1$ norm of the projected observation vectors over all observations. However, in contrast to the conventional $L_2$ norm based PCA, the solutions of \textsf{P1} and \textsf{P2} with the $L_1$ norm might not be same. The work in \cite{Galpin-Hawkins:87} studies $L_1$-norm based covariance matrix estimation, while the works in \cite{Choulakian:06,Croux2005,Kwak:08,Li-Chen1985} and \cite{Nie-etal:11} directly consider \textsf{P2}. The algorithm in \cite{Kwak:08} finds a local optimal solution by sequentially obtaining one PC that is orthogonal to the previously obtained PCs based on a greedy strategy. Recently, the work in \cite{Nie-etal:11} extended the algorithm in \cite{Kwak:08} using a non-greedy strategy. The works in \cite{Markopoulos:14} and \cite{McCoy:11} show that \textsf{P2} with one PC is NP-hard when the number of observations and attributes are jointly arbitrarily large. The work in \cite{Markopoulos:14} provides a polynomial algorithm when the number of attributes is fixed.

As the objective functions are different, solving for \textsf{P1} and \textsf{P2} with different norms give different solutions in terms of the signs and order of the PCs. In Figure \ref{fig:cancer}, we present heat maps of PCs obtained by solving $L_2$-PCA, \textsf{P1} with the $L_1$ norm, and \textsf{P2} with the $L_1$ norm with $p = 5$ for data set \textit{cancer\_2} presented in Table \ref{tab:uca_data} and used in the experiment in Section \ref{section_exp_uci}. The three heat maps represent the matrices of PCs of $L_2$-PCA (left), \textsf{P1} with the $L_1$ norm (center), and \textsf{P2} with the $L_1$ norm (right). The rows and columns of each matrix represent original attributes and PCs, respectively. The blue, white, and red cells represent the intensity of positive, zero, and negative values. Note that both $L_1$-PCA variations give Attr 1 a large negative loading in either of the first or second PCs, whereas $L_2$-PCA gives Attr 1 a large positive loading in the third PC. Attr 9 has a large negative loading for the fifth PC of \textsf{P1} with $L_1$ norm, while $L_2$-PCA gives Attr 9 a large positive loading in the second PC. 

\begin{figure}[h]
\centering
\includegraphics[width=0.48\textwidth]{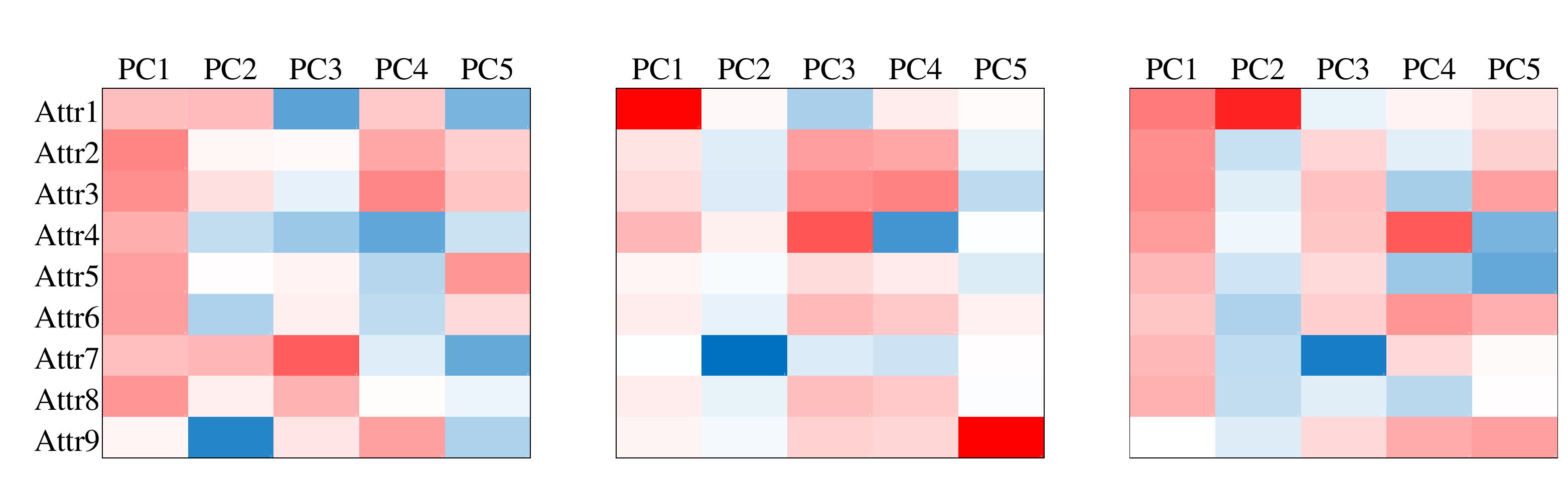}
\caption{Heat maps of 5 principal components by $L_2$-PCA (left), \textsf{P1} with $L_1$ norm (center), and \textsf{P2} with $L_1$ norm (right) for \textit{cancer\_2} data set }
\label{fig:cancer}
\end{figure}

In this paper, we propose two iterative algorithms for \textsf{P1} with the $L_1$ norm, and provide analytical convergence results. Although we propose iterative algorithms, our focus is the $L_1$ objective function and thus our algorithms do not directly compare with iterative algorithms for the standard $L_2$-PCA such as algorithms based on EM \cite{Roweis:98} or NIPALS \cite{Geladi-Kowalski:86}. For \textsf{P1} with the $L_1$ norm, the first proposed algorithm, the exact reweighted version, is based on iteratively reweighted least squares (IRLS) that gives a weight to each observation. The weighted least square problem is then converted into the standard $L_2$-PCA problem with a weighted data matrix, and the algorithm iterates over different weights. We show that the algorithm gives convergent weights and the weighted data matrices have convergent eigenvalues. We also propose an approximate version in order to speed up the algorithm and show the convergent eigenvectors. The work in \cite{Daubechies:2010} provides an IRLS algorithm for sparse recovery. The reader is referred to \cite{Jorgensen:06} for a review of IRLS applied in different contexts. Recently, the work in \cite{Wen-Yin} studied the minimization of a differentiable function of an orthogonal matrix. In the computational experiment, we compare our algorithms with benchmark algorithms in \cite{Brooks-etal:12, Ke-Kanade:05, Kwak:08,Nie-etal:11}. The experiment shows that our algorithm for \textsf{P1} regularly outperforms the benchmark algorithms with respect to the solution quality and its computational time is of the same order as the fastest among the other four. Even though $L_1$-PCA can be used for building robust PCs and is an alternative to other robust PCA approaches such as the work in \cite{candes:11}, we limit the comparison for the $L_1$-PCA objective functions introduced in Section \ref{section_algorithms}, as our goal is to directly optimize the $L_1$ norms for \textsf{P1}.

Our contributions can be summarized as follows.
\begin{enumerate}[noitemsep]
\item For \textsf{P1} with the $L_1$ norm, we propose the first IRLS algorithm in the literature, and show the convergence of the eigenvalues of the weighted matrices. The algorithm directly minimizes the reconstruction error, while the other benchmark algorithms primarily try to find the optimal subspace. An additional advantage of our algorithm is that it uses an $L_2$-PCA algorithm as a black box. Hence, by using a more scalable algorithm, the practical time complexity of our algorithm can further be reduced.
\item We propose an approximate version to speed up the algorithm and to guarantee convergent eigenvectors. The  difference is that the approximate version uses a formula to update the eigenpairs when the changes in the weighted data matrices become relatively small.
\item The results of the computational experiment show that the proposed algorithms for \textsf{P1} outperform the benchmark algorithms in most cases.
\end{enumerate}

The rest of the paper is organized as follows. In Section \ref{section_algorithms}, we present the algorithms and show all analytic results. Then in Section \ref{section_computational_experiment}, the computational experiment is presented.

\section{Algorithms for L1 PCA}
\label{section_algorithms}

In this section, we present the algorithms for \textsf{P1} and show the underlying analytic results. We use the following notation.
\begin{enumerate}[noitemsep]
\item[-] $n$: number of observations of the data
\item[-] $m$: number of attributes of the data
\item[-] $p$: number of principal components (target dimension)
\item[-] $I = \{1,\cdots, n \}$: index set of the observations
\item[-] $J = \{1,\cdots,m\}$: index set of the attributes
\item[-] $P =\{1,\cdots,p\}$: index set of the PCs
\item[-] $A \in \mathbb{R}^{n \times m} $: data matrix with elements $a_{ij}$ for $i \in I, j \in J$
\item[-] $X \in \mathbb{R}^{m \times p}$: principal components matrix with elements $x_{jk}$ for $j \in J, k \in P$
\item[-] $Y \in \mathbb{R}^{n \times p}$: projected data matrix with elements $y_{ik}$ for $i \in I, k \in P$, defined as $Y = AX$
\item[-] $E \in \mathbb{R}^{n \times m}$: reconstruction error matrix with elements $e_{ij}$ for $i \in I, j \in J$, defined as $E= A - Y X^{\top}$
\item[-] $I_k \in \mathbb{R}^{k \times k}$: $k$ by $k$ identity matrix 
\end{enumerate}
For a matrix $R \in \mathbb{R}^{n_r \times m_r}$ with elements $r_{ij}$, $i=1,\cdots, n_r$, $j=1,\cdots,m_r$, we denote by  
\begin{center}
$\| R \|_{F} = \sqrt{\sum_{i=1}^{n_r} \sum_{j=1}^{m_r} r_{ij}^2}$
\end{center}
the Frobenius norm.

The conventional PCA problem, \textsf{P1} with the $L_2$ norm, can be written as
\begin{equation}
\label{definition_L2PCA}
\min_{X \in \mathbb{R}^{m \times p}, X^{\top} X  = I_p} \| A - AX X^{\top} \|_{F}^2.
\end{equation}
Note that $X^{\top} X  = I_p$ is different from $X X^{\top} \in \mathbb{R}^{m \times m}$ in the objective function. 

We consider \eqref{definition_L2PCA} with the $L_1$ norm instead of the $L_2$ norm in the objective function. The resulting \textsf{P1} problem is written as
\begin{equation}
\label{definition_L1PCA_minError}
\min_{X \in \mathbb{R}^{m \times p}} \sum_{i \in I} \sum_{j \in J} |e_{ij}| \mbox{ s.t. } X^{\top} X  = I_p, E = A - AXX^{\top}.
\end{equation}

Next we present an iterative algorithm for \eqref{definition_L1PCA_minError} to minimize the reconstruction error. Instead of solving \eqref{definition_L1PCA_minError} directly, we iteratively solve a weighted version of \eqref{definition_L2PCA} by giving a different weight for each observation. 

We rewrite \eqref{definition_L1PCA_minError} in the following non-matrix form.
\begin{subequations}
\label{formulation_PCA_L1_minError}
\begin{align}
\min \quad & \displaystyle \sum_{i \in I} \sum_{j \in J} |e_{ij}| \label{formulation_PCA_L1_minError_a} \\
s.t.\quad& \displaystyle y_{ik} = \sum_{j \in J} a_{ij} x_{jk}, \qquad \qquad i \in I, k \in P, \label{formulation_PCA_L1_minError_b} \\
& \displaystyle e_{ij} = a_{ij} - \sum_{k \in P} y_{ik} x_{jk}, \quad i \in I, j \in J, \label{formulation_PCA_L1_minError_c} \\
& \displaystyle \sum_{j \in J} x_{jk} x_{jq} = 0, \qquad \qquad k \in P,q \in P, k < q,   \label{formulation_PCA_L1_minError_d}\\
& \displaystyle \sum_{j \in J} x_{jk} x_{jk} = 1, \qquad \qquad k \in P, \label{formulation_PCA_L1_minError_e}\\
& E, X, Y \mbox{ unconstrained} \label{formulation_PCA_L1_minError_f}.
\end{align}
\end{subequations}
Note that the corresponding $L_2$-PCA problem \eqref{definition_L2PCA} can be also written as
\begin{equation}
\label{formulation_PCA_L2_minError}
\displaystyle \min_{E,X,Y} \sum_{i \in I} \sum_{j \in J} e_{ij}^2 \quad \mbox{s.t.} \quad \eqref{formulation_PCA_L1_minError_b} - \eqref{formulation_PCA_L1_minError_f},
\end{equation} 
since the only difference between \eqref{definition_L2PCA} and \eqref{definition_L1PCA_minError} is the objective function. However, there is no known algorithm that solves \eqref{formulation_PCA_L1_minError} optimally, whereas \eqref{formulation_PCA_L2_minError} can be solved by SVD or EVD. Hence, we want to take advantage of the fact that \eqref{formulation_PCA_L2_minError} can be optimally solved. 

Let us consider a weighted version of \eqref{formulation_PCA_L2_minError}:
\begin{equation}
\label{formulation_PCA_weighted_L2_minError_final}
\displaystyle g(w) = \min_{E,X,Y} \sum_{i \in I} \sum_{j \in J} w_{i}e_{ij}^2 \quad \mbox{s.t.} \quad \eqref{formulation_PCA_L1_minError_b} - \eqref{formulation_PCA_L1_minError_f},
\end{equation}
with $w_i >0$ for every $i$ in $I$. Note that \eqref{formulation_PCA_L2_minError} and \eqref{formulation_PCA_weighted_L2_minError_final} are equivalent when $w_i=1$ for all $i$ in $I$. However, solving \eqref{formulation_PCA_weighted_L2_minError_final} with non-constant $w_i$'s is not easy in its original form due to the orthogonality constraint. Instead, let us define weighted data matrix $\bar{A} \in \mathbb{R}^{n \times m}$ with each element defined as
\begin{equation}
\label{definition_a_bar}
\bar{a}_{ij} = \sqrt{w_{i}} a_{ij},  \mbox{ for }  i \in I, j \in J.
\end{equation}
In the following proposition, we show that an optimal solution to \eqref{formulation_PCA_weighted_L2_minError_final} can be obtained by SVD of $\bar{A}$.

\begin{proposition}
\label{proposition_weighted_PCA_equivalence}
Solving \eqref{formulation_PCA_weighted_L2_minError_final} with $A$ is equivalent to solving \eqref{formulation_PCA_L2_minError} with $\bar{A}$.
\end{proposition}
\begin{proof}
Let $(\bar{E}, \bar{X}, \bar{Y})$ be a solution to \eqref{formulation_PCA_L2_minError} with $\bar{A}$. We claim that, for any $(\bar{E}, \bar{X}, \bar{Y})$, there exists $(E,X,Y)$ with the same objective function value for \eqref{formulation_PCA_weighted_L2_minError_final} with $A$, and vice versa. Let $y_{ik} = \frac{1}{\sqrt{w_i}} \bar{y}_{ik}$, $e_{ij} = \frac{1}{\sqrt{w_i}} \bar{e}_{ij}$, and $x_{jk} = \bar{x}_{jk}$. We derive
\vspace{0.2cm}

\begin{tabular}{p{8cm}}
$y_{ik} =  \frac{\bar{y}_{ik}}{\sqrt{w_i}} $ $= \frac{\sum_{j \in J} \bar{a}_{ij} \bar{x}_{jk}}{\sqrt{w_i}} = \frac{\sum_{j \in J} \sqrt{w_i} a_{ij} \bar{x}_{jk}}{\sqrt{w_i}}$,\\
$e_{ij} =  \frac{\bar{e}_{ij}}{\sqrt{w_i}}  =  \frac{\bar{a}_{ij} - \sum_{k \in P} \bar{y}_{ik} \bar{x}_{jk}}{\sqrt{w_i}} = a_{ij} - \sum_{k \in P} y_{ik} \bar{x}_{jk}$ \\
\hspace{0.4cm} $= a_{ij} - \sum_{k \in P} y_{ik} x_{jk} $,\\
$\sum_{i=1}^{n} \sum_{j \in J} \bar{e}_{ij}^2 = \sum_{i=1}^{n} \sum_{j \in J} w_i e_{ij}^2$.
\end{tabular}
\vspace{0.2cm}

Further, since $\bar{x}_{jk} = x_{jk}$ for all $j$ and $k$, orthonormality constraints \eqref{formulation_PCA_L1_minError_d} and \eqref{formulation_PCA_L1_minError_e} are automatically satisfied. Hence, solving \eqref{formulation_PCA_L2_minError} with $\bar{A}$ is equivalent to solving \eqref{formulation_PCA_weighted_L2_minError_final} with $A$.
\end{proof}

Since now we know that \eqref{formulation_PCA_weighted_L2_minError_final} can be solved optimally, the remaining task is to define appropriate weights that give a good solution to \eqref{formulation_PCA_L1_minError}. We first provide intuition behind our choices.

Let $(E^*,X^*,Y^*)$ be an optimal solution to \eqref{formulation_PCA_L1_minError} and let
\begin{equation}
\label{weight_for_opt_error}
w_i^* = \left \{
	\begin{array}{ll}
		\frac{\sum_{j \in J} |e_{ij}^*| }{\sum_{j \in J} (e_{ij}^*)^2},  &  \mbox{if } \sum_{j \in J}  (e_{ij}^*)^2 > 0,\\ 
		M, &  \mbox{if } \sum_{j \in J}  (e_{ij}^*)^2 = 0,
	\end{array}
\right.
\end{equation}
be weights defined based on $E^*$, where $M$ is a large number. Note that the value of $M$, for the case $ \sum_{j \in J}  (e_{ij}^*)^2 = 0$, does not affect the value of $\sum_{i \in I} w_i^* \sum_{j \in J}  (e_{ij}^*)^2$ because $\sum_{j \in J}  (e_{ij}^*)^2=0$ for the corresponding observation. However, considering the fact that we want to give less weight to the outliers in order to reduce their effect on the objective function, it is reasonable to assign a big number to the observations with zero error.

With $w^*$ defined in \eqref{weight_for_opt_error}, it is trivial to show
\begin{equation}
\label{eqn_property_f_star}
\sum_{i \in I} \sum_{j \in J} |e_{ij}^*|  = \sum_{i \in I} w_i^* \sum_{j \in J}  (e_{ij}^*)^2  \mbox{ and }  g(w^*) \leq \sum_{i \in I} \sum_{j \in J} |e_{ij}^*|. 
\end{equation}
The equality in \eqref{eqn_property_f_star} implies that, given $w^*$ and $E^*$, the objective function value of \eqref{formulation_PCA_L1_minError} and \eqref{formulation_PCA_weighted_L2_minError_final} are equal. The inequality in \eqref{eqn_property_f_star} implies that, given $w^*$, the objective function value of \eqref{formulation_PCA_weighted_L2_minError_final} gives a lower bound on the optimal objective function value of \eqref{formulation_PCA_L1_minError}. Hence, we aim to minimize the objective function of \eqref{formulation_PCA_L1_minError} by solving \eqref{formulation_PCA_weighted_L2_minError_final}, hoping $g(w^*)$ and $\sum_{i \in I} \sum_{j \in J} |e_{ij}^* |$ are not far from each other. 

The equality and inequality in \eqref{eqn_property_f_star} give motivation to use a weight formula similar to \eqref{weight_for_opt_error}. Before presenting the weight update formula and the algorithm, let us define the following notation for the algorithm.

\begin{enumerate}[noitemsep]
\item[-] $t$: current iteration
\item[-] $w^t \in \mathbb{R}^{n \times 1}$: weight vector used in iteration $t$ with elements $w_i^t$ for $i \in I$
\item[-] $W_t \in \mathbb{R}^{n \times n}$: diagonal matrix in iteration $t$ with $\sqrt{w_i^t}$'s on the diagonal
\item[-] $A_t \in \mathbb{R}^{n \times m}$: weighted data matrix defined in \eqref{definition_a_bar} with $w^t$ in iteration $t$, defined as $A_t = W_t A$
\item[-] $X_t \in \mathbb{R}^{m \times p}$: the principal component matrix obtained by SVD of $A_t$ in iteration $t$
\item[-] $E_t \in \mathbb{R}^{n \times m}$: reconstruction error matrix in iteration $t$ with elements $e_{ij}^t$ for $i \in I, j \in J$, defined as $E_t = A - AX_t X_t^{\top}$
\item[-] $L2PCA(A_t,p)$: subroutine that returns $X_t \in \mathbb{R}^{m \times p}$ by solving \eqref{formulation_PCA_L2_minError} with $A_t$
\item[-] $F(X_t)$: objective function value of $X_t$ for \eqref{formulation_PCA_L1_minError}, defined as $F(X_t) = \sum_{i \in I} \sum_{j \in J} |e_{ij}^t|$
\item[-] $F^{best}$: current best objective function value
\item[-] $X^{best}$: principal component matrix associated with $F^{best}$
\end{enumerate}
Note that $X_t$ is obtained by SVD of $A_t$, but $E_t = A - AX_t X_t^{\top}$ is based on $A$ and is different from $A_t - A_t X_t X_t^{\top}$. 

Motivated by \eqref{weight_for_opt_error}, for iteration $t+1$, we define 
\begin{equation}
\label{def_temp_weight}
u_{i}^{t+1} = \left \{
	\begin{array}{ll}
		\frac{\sum_{j \in J} | e_{ij}^t | }{\sum_{j \in J}  (e_{ij}^t)^2},  &  \mbox{if } \sum_{j \in J}  (e_{ij}^t)^2 > 0,\\ 
		M_t, &  \mbox{if } \sum_{j \in J}  (e_{ij}^t)^2 = 0,
	\end{array}
\right.
\end{equation}
where $M_t = \max_{i \in I_t^+} \frac{\sum_{j \in J} | e_{ij}^t | }{\sum_{j \in J}  (e_{ij}^t)^2}$ is the largest weight among the observations in $I_t^+ = \{ i \in I | \sum_{j \in J}  (e_{ij}^t)^2 > 0 \}$. Using $u^{t+1}$ in \eqref{def_temp_weight} for the weights is natural and we empirically observe that the algorithm is convergent. However, it is not trivial to show the convergence with $u^{t+1}$ for \eqref{formulation_PCA_weighted_L2_minError_final}. Hence, in order to show the convergence of the algorithm, we present a modified update formula based on $u^{t+1}$. 
\begin{equation}
\label{def_weight}
w_{i}^{t+1} = \left \{
	\begin{array}{ll}
		w_i^t (1-\beta^t),  &  \mbox{if } u_i^{t+1} < w_i^{t} (1-\beta^t),\\  
		u_i^{t+1}, &  \mbox{if } w_i^{t} (1-\beta^t) \leq u_i^{t+1} \leq w_i^{t} (1+\beta^t),\\ 
		w_i^t (1+\beta^t),  &  \mbox{if } u_i^{t+1} > w_i^{t} (1+\beta^t),\\ 
	\end{array}
\right.
\end{equation}
where $\beta \in (0,1)$. Note that $\beta^t$ is the $\beta$ to the power of $t$ and is different from other superscript-containing notations such as $w_i^t$ or $u_i^t$. The role of \eqref{def_weight} is to enable bounds of the change for $w^{t+1}$ from $w^t$. If $u_i^{t+1}$ is too small compared to $w_i^t$, then $w_i^{t+1}$ is assigned a value between $u_i^{t+1}$ and $w_i^t$. If $u_i^{t+1}$ is too large compare to $w_i^t$, then $w_i^{t+1}$ obtains a value between $u_i^{t+1}$ and $w_i^t$. Otherwise, $w_i^{t+1}$ follows the weight formula in \eqref{def_temp_weight}. Given $\beta \in (0,1)$, we have $\lim_{t \rightarrow \infty} \beta^t = 0$, which implies $\lim_{t \rightarrow \infty} w_i^t - w_i^{t+1} = 0$. Further, since $u$ and $w$ are bounded above and below, we can show that $w^t$ is convergent. By setting $\beta$ close to 1, we would have $w_i^{t+1} = u_i^{t+1}$ in most cases, as $1-\beta^t$ and $1+\beta^t$ are close to 0 and 2 for small values of $t$, i.e., early iterations. From all these facts and by using elementary mathematics, the following lemma follows.

\begin{lemma}
\label{lemma_convergent_weight}
With $w^t$ defined in \eqref{def_weight}, $w^t$ and $A_t$ are convergent in $t$.
\end{lemma}

We present the overall algorithmic framework in Algorithm \ref{algo_IRLSP}. The algorithm requires data matrix $A$, target dimension $p$, and tolerance parameter $\varepsilon$ as input. After the initialization of weights and the best objective function value $F^{best}$ in Step \ref{algo_IRLSP_line1}, the while loop is executed until $w^t$ and $w^{t+1}$ are close enough. In each iteration of the while loop, $A_t$ is constructed based on $w_t$ and $X_t$ is obtained by SVD of $A_t$ (Steps \ref{algo_IRLSP_line3} and \ref{algo_IRLSP_line4}). If $X_t$ gives a lower objective function value than $X^{best}$, then $X^{best}$ and $F^{best}$ are updated. Recall that $X_t$ is obtained by using $A_t$, but $F(X_t)$ uses the original data matrix $A$. Each iteration ends after the update of weights in Step \ref{algo_IRLSP_line6}. Observe that the termination criteria in Step \ref{algo_IRLSP_line2} solely depends on the convergence of $w_t$. Hence, the algorithm terminates in a finite number of iterations.

\begin{algorithm}[ht]
\caption{wPCA (Weight-based algorithm for $L_1$-PCA)}        
\label{algo_IRLSP}                           
\begin{algorithmic}[1]    
\REQUIRE $A$ (data), $p$ (target dimension), $\varepsilon$ (tolerance), $\beta$ \\
\ENSURE principal components $X^{best} \in \mathbb{R}^{m \times p}$\\
\vspace{0.1cm}
\STATE $t \gets 1$, $w_i^0 \gets 2$, $w_i^1 \gets 1$, $F^{best} \gets \infty$ \label{algo_IRLSP_line1}
\WHILE{$\| w^t - w^{t-1} \|_1 > \varepsilon$} \label{algo_IRLSP_line2}
	\STATE set $A_t$ based on $w^t$ \label{algo_IRLSP_line3}
	\STATE $X_t \gets L2PCA(A_t,p)$ \label{algo_IRLSP_line4}
	\STATE \textbf{if} $F(X_t)$ $<$ $F^{best}$ \textbf{then} $X^{best} \gets X_t$, $F^{best} \gets F(X_t)$ \label{algo_IRLSP_line5}
	\STATE update $w_{t+1}$ by using \eqref{def_weight} given $\beta$ \label{algo_IRLSP_line6}
	\STATE $t \gets t+1$  \label{algo_IRLSP_line7}
\ENDWHILE
\end{algorithmic}
\end{algorithm}

\begin{lemma}
\label{lemma_convergent_B_t}
Eigenvalues of $A_t^{\top} A_t$ are convergent in $t$.
\end{lemma}
\begin{proof}
Recall that the weights $w_t$ and weighted matrix $A_t$ are convergent, which also implies $A_t^{\top} A_t$ is also convergent. Since the eigenvalues of symmetric matrices are point-wise convergent if the matrices are convergent \cite{Lax:2007}, it is trivial to see that the eigenvalues of $A_t^{\top} A_t$ are convergent. 
\end{proof}

Hence, Algorithm \ref{algo_IRLSP} gives convergent eigenvalues. Although eigenvalues are convergent, it is not trivial to show the convergence of $X_t$. This is because even a slight change in a matrix can cause a change in an eigenvector, and eigenvalue-eigenvector pairs are not unique.

In order to provide convergent eigenpairs and accelerate the algorithm for large scale data, we use the first order eigenpair approximation formula from \cite{Shmueli:12}. Let $(X_{t-1}^i,\lambda_{t-1}^i)$ be an approximate eigenpair of $A_{t-1}$ and $A_{t} = A_{t-1} + \Delta_t$. Then, the approximate eigenpairs of $A_{t}$ can be obtained by
\begin{eqnarray}
\lambda_{t}^i = \lambda_{t-1}^i + (X_{t-1}^i)^{\top} \Delta_t X_{t-1}^i, \label{formula_ep1}\\
X_{t}^i = X_{t-1}^i + \sum_{j \neq i} \Big( \frac{(X_{t-1}^j)^{\top} \Delta_t X_{t-1}^i}{\lambda_{t-1}^i - \lambda_{t-1}^j} \Big) X_{t-1}^j, \label{formula_ep2}
\end{eqnarray}
using the formula in \cite{Shmueli:12}. The error is of the order of $o(\|\Delta_t \|^2)$. Let \textit{L2PCA\_Approx} be a function that returns principal components by formula \eqref{formula_ep1} - \eqref{formula_ep2}. The modified algorithm is presented in Algorithm \ref{algo_IRLSP_fast}.

\begin{algorithm}[ht]
\caption{awPCA (Approximated weight-based algorithm for $L_1$-PCA)}        
\label{algo_IRLSP_fast}                           
\begin{algorithmic}[1]    
\REQUIRE $A$ (data), $p$ (target dimension), $\varepsilon$ (tolerance), $\beta$, $\gamma$ \\
\ENSURE principal components $X^{best} \in \mathbb{R}^{m \times p}$\\
\vspace{0.1cm}
\STATE $t \gets 1$, $w_i^0 \gets 2$, $w_i^1 \gets 1$, $F^{best} \gets \infty$ 
\WHILE{$\| w^t - w^{t-1} \|_1 > \varepsilon$} 
	\STATE set $A_t$ based on $w^t$, $\Delta_t \gets A_t - A_{t-1}$
	\STATE \textbf{If} $\|\Delta_t\|^2  > \gamma \cdot \|A\|^2$ \textbf{then} $X_t \gets L2PCA(A_t,p)$ \label{algo_IRLSP_fast_line4}
	\STATE \textbf{Else} $(X_t,\lambda_t) \gets L2PCA\_approx(X_{t-1}, \lambda_{t-1}, p)$ \label{algo_IRLSP_fast_line5}
	\STATE \textbf{if} $F(X_t)$ $<$ $F^{best}$ \textbf{then} $X^{best} \gets X_t$, $F^{best} \gets F(X_t)$ 
	\STATE update $w_{t+1}$ by using \eqref{def_weight} given $\beta$
	\STATE $t \gets t+1$  
\ENDWHILE
\end{algorithmic}
\end{algorithm} 

The difference is only in Lines \ref{algo_IRLSP_fast_line4} and \ref{algo_IRLSP_fast_line5}. If the change in $A_t$ is large (greater than $\gamma \cdot \|A\|^2$), we use the original procedure \textit{L2PCA}. If the change in $A_t$ is small (less than or equal to $\gamma \cdot \|A\|^2$), then we use the update formula \eqref{formula_ep1} and \eqref{formula_ep2}. Algorithm \ref{algo_IRLSP_fast} has the following convergence result. 

\begin{proposition}
\label{proposition_convergent_eigenpairs}
Eigenpairs of $A_t^{\top} A_t$ in Algorithm \ref{algo_IRLSP_fast} are convergent in $t$.
\end{proposition}
\begin{proof}

Note that the convergence of $\Delta_t$ in Lemma \ref{lemma_convergent_B_t} does not depend on how the eigenvectors are obtained and thus it holds whether we execute Lines \ref{algo_IRLSP_fast_line4} or \ref{algo_IRLSP_fast_line5} in each iteration. Hence, we have $\displaystyle \lim_{t \rightarrow \infty} \Delta_t = 0$ in Algorithm \ref{algo_IRLSP_fast}.

Since $\Delta_t$ converges to zero, after a certain number of iterations $\bar{t}$, we have $\|\Delta_t\|^2  < \gamma \cdot \|A\|^2$ for all $t > \bar{t}$. From such large $\bar{t}$, the approximation rule applies. In \cite{Shmueli:12}, it is shown that 
\begin{center}
$\bar{\lambda}_{t}^i = \bar{\lambda}_{t-1}^i + (\bar{X}_{t-1}^i)^{\top} \Delta_t \bar{X}_{t-1}^i  + o(\|\Delta_t \|^2),$\\
$\bar{X}_{t}^i = \bar{X}_{t-1}^i + \sum_{j \neq i} \Big( \frac{(\bar{X}_{t-1}^j)^{\top} \Delta_t \bar{X}_{t-1}^i}{\bar{\lambda}_{t-1}^i - \bar{\lambda}_{t-1}^j} \Big) \bar{X}_{t-1}^j  + o(\|\Delta_t \|^2) ,$ 
\end{center}
when $(\bar{X}_{t-1},\bar{\lambda}_{t-1})$ are the exact eigenpairs. Since all of the terms in the formula are bounded and $\displaystyle \lim_{t \rightarrow \infty} \Delta_t = 0$, we conclude that the eigenpairs of $A_t^{\top} A_t$ are convergent.
\end{proof}

\section{Computational Experiment}
\label{section_computational_experiment}

We compare the performance of the proposed and benchmark algorithms for varying instance sizes $(m,n)$ and number of PCs $(p)$. All experiments were performed on a personal computer with 8 GB RAM and Intel Core i7 (2.40GHz dual core). We implement Algorithms \ref{algo_IRLSP} and \ref{algo_IRLSP_fast} in R \cite{Rstat}, which we denote as \textsf{wPCA} and \textsf{awPCA}, respectively. The R script for \textsf{wPCA} and \textsf{awPCA} is available on a web site \footnote[1]{\url{http://dynresmanagement.com/uploads/3/3/2/9/3329212/wl1pca.zip}}. For the \textsf{awPCA} implementation, we use condition $\| w^t - w^{t-1} \|_1 > \gamma \cdot \| w^t \|_1$ instead of $\|\Delta_t\|^2  > \gamma \cdot \|A\|^2$, to avoid unnecessary calculation of $\|\Delta_t\|^2$ in Line \ref{algo_IRLSP_line4} of Algorithm \ref{algo_IRLSP_fast}. The weight-based condition is similar to the original condition as the difference in the weight captures $\Delta_t$. For the experiment, we use parameters $\varepsilon = 0.001$ and $\beta = 0.99$ for \textsf{wPCA} and \textsf{awPCA} and $\gamma = 0.1$ for \textsf{awPCA}, where the parameters are tuned based on pilot runs to balance the solution quality and execution time. We also set up maximum number of iterations to 200. We compare our algorithms with the algorithms in \cite{Brooks-etal:12, Ke-Kanade:05, Kwak:08, Nie-etal:11}. The work in \cite{Brooks-etal:12b} provides R implementations of the algorithms in \cite{Brooks-etal:12, Ke-Kanade:05, Kwak:08}, which we denote as \textsf{Brooks}, \textsf{Ke}, \textsf{Kwak}, respectively. We implement the algorithm in \cite{Nie-etal:11} in R, which we denote as \textsf{Nie}.

Although algorithms \textsf{Ke} and \textsf{Brooks} are for \eqref{definition_L1PCA_minError} and \textsf{Kwak} and \textsf{Nie} are for the $L_1$ norm version of \textsf{P2}, we evaluate the objective function value for all benchmark algorithms and compare them against our algorithms. Especially, \textsf{Kwak} and \textsf{Nie}, which solve different $L_1$-PCA problem, are included because
\begin{itemize}
\item  we observed that \textsf{Kwak} and \textsf{Nie} are better than \textsf{Ke} and \textsf{Brooks} for \eqref{definition_L1PCA_minError} for some instances, and
\item we found that \textsf{Kwak} and \textsf{Nie} are more scalable and solve larger data sets in a reasonable time in the experiment.
\end{itemize}
Therefore, we include \textsf{Kwak} and \textsf{Nie} for the comparison for solving \textsf{P1}.

It is worth to note that \textsf{Ke} and \textsf{Brooks} try to find the best fitting subspace, where definition of $E$ in \eqref{definition_L1PCA_minError} is replaced by $E = A-UX^{\top}$ and $U \in R^{n \times p}$. The optimal solutions of the two formulations may be different despite both minimizing the $L_1$ distance from the original data.

Let $F_{algo}$ represent the objective function value obtained by $algo \in \{$\textsf{Ke}, \textsf{Brooks}, \textsf{Kwak}, \textsf{Nie}, \textsf{wPCA}, \textsf{awPCA}$\}$, with respect to \eqref{definition_L1PCA_minError}. For the comparison purposes for \textsf{awPCA}, we use the gap from the best objective function value defined as
\begin{center}
\begin{tabular}{l}
$\Delta_{algo} = \min \Big\{ \frac{F_{algo}}{\min\{F_{\mbox{\scriptsize{\textsf{awPCA}}}}, F_{\mbox{\scriptsize{\textsf{Ke}}}}, F_{\mbox{\scriptsize{\textsf{Brooks}}}}, F_{\mbox{\scriptsize{\textsf{Kwak}}}}, F_{\mbox{\scriptsize{\textsf{Nie}}}}\} } - 1, 1 \Big\}$, 
\end{tabular}
\end{center}
\noindent for each $algo \in \{$\textsf{awPCA}, \textsf{Ke}, \textsf{Brooks}, \textsf{Kwak}, \textsf{Nie}$\}$. Similarly, for \textsf{wPCA}, we define
\begin{center}
\begin{tabular}{l}
$\Delta_{algo} = \min \Big\{  \frac{F_{algo}}{\min\{F_{\mbox{\scriptsize{\textsf{wPCA}}}}, F_{\mbox{\scriptsize{\textsf{Ke}}}}, F_{\mbox{\scriptsize{\textsf{Brooks}}}}, F_{\mbox{\scriptsize{\textsf{Kwak}}}}, F_{\mbox{\scriptsize{\textsf{Nie}}}}\} } -1, 1 \Big\}$, 
\end{tabular}
\end{center}
\noindent for each $algo \in \{$\textsf{wPCA}, \textsf{Ke}, \textsf{Brooks}, \textsf{Kwak}, \textsf{Nie}$\}$. Note that $\Delta_{algo}$ represents the gap between \textit{algo} and the best of all algorithms. Note also that we set up an upper bound of 1 for $\Delta_{algo}$. Hence, if the gap is larger than 1 (or 100\%), then $\Delta_{algo}$ is assigned value of 1 (or 100\%).

For all of the instances used in the experiment, we first standardize each column and deal exclusively with the standardized data. Hence, the reconstruction errors are also calculated based on the standardized data.

In the computational experiment, we observed that $\Delta_{\mbox{\scriptsize{\textsf{awPCA}}}}$ and $\Delta_{\mbox{\scriptsize{\textsf{wPCA}}}}$ are very similar while the execution time of \textsf{awPCA} is much faster. Hence, in this section, we first focus on presenting the performance of \textsf{awPCA} against the benchmark algorithms and after on comparing the difference between \textsf{wPCA} and \textsf{awPCA}.

The rest of the section is organized as follows. In Section \ref{section_exp_instance}, we present synthetic instance generation procedures and explain the instances from the UCI Machine Learning Repository \cite{Bache+Lichman:2013}. In Sections \ref{section_exp_synthetic} and \ref{section_exp_uci}, we present the performance of \textsf{awPCA} for the synthetic and UCI instances, respectively. In Section \ref{section_exp_compare_two}, we compare the performance of \textsf{wPCA} and \textsf{awPCA} for the UCI instances.

\subsection{Instances}
\label{section_exp_instance}

\subsubsection{Synthetic Instances}

In order to provide a systematic analysis, we generate synthetic instances with presence of outliers and various $(m,n,r)$, where $m \in \{20,50\}$, $n \in \{100,300\}$, and $r \in \{0,0.1,0.2,0.3\}$. For each $(m,n,r)$, we generate 5 distinguished instances. Hence, we have a total of 80 generated instances. The synthetic instances used in the experiment are available on a web site \footnote[2]{\url{http://dynresmanagement.com/uploads/3/3/2/9/3329212/pca_instance_park_klabjan.zip}}. The detailed algorithm is presented at the end of this section. In the generation procedure, $r \%$ of observations are generated to have a higher variance than the remaining normal observations. The instance generation algorithm needs additional parameter $q$ (target rank), which we fix to 10 for the instances we generated in this experiment. Hence, the instances we use in the experiment are likely to have rank equal to 10. We consider different $p$ values, where $p \in \{8,9,10,11,12\}$. Given that $q=10$, we select $p$ values around 10.

The purpose of the instance generation algorithm is to generate instances with some of the observations as outliers, so that $L_1$-PCA solutions are more likely to be away from $L_2$-PCA solutions. In order to generate instances, we use the procedure described in \cite{Ke-Kanade:05} with a slight modification. The instances used in the experiment in \cite{Ke-Kanade:05} have fixed parameters and constant valued outliers to simulate data loss. To check the performance of the algorithms over various parameters and different (non-constant) patterns of outliers, we generate our own instances. In the generation procedure of \cite{Ke-Kanade:05}, a matrix with a small fixed rank is generated and then extremely large constant values randomly replace the original data matrix. In their instances, outliers have the same value, which can be interpreted as data loss, but they do not consider outliers due to incorrect measurements or cases with only a few observations with outliers. Our procedure addresses all these cases.

We present the procedure in Algorithm \ref{algo_pca_instance_generation}.

\begin{algorithm}[ht]
\caption{PCA instance generation}        
\label{algo_pca_instance_generation}                           
\begin{algorithmic}[1]    
\REQUIRE $m$, $n$, $q$ (target rank), $r$ (\% outliers) \\
\ENSURE $A \in \mathbb{R}^{n \times m}$\\
\vspace{0.1cm}
\STATE Generate random matrix $P = [p_{ij}]\in \mathbb{R}^{n \times m}$ with $p_{ij} \sim U(-100,100)$ \label{algo_pca_instance_generation_line_1}
\STATE \textbf{for} each row  \label{algo_pca_instance_generation_line_2}
\STATE \quad  Generate random number $u_1 \sim U(0,1)$  \label{algo_pca_instance_generation_line_3}
\STATE \quad  \textbf{if} $u_1<r$ \label{algo_pca_instance_generation_line_4}
\STATE \quad \quad   \textbf{for} each column $j \leq q$ \label{algo_pca_instance_generation_line_5}
\STATE \quad \quad \quad    Generate $u_2 \sim U(0,1)$ \label{algo_pca_instance_generation_line_6}
\STATE \quad \quad \quad    \textbf{if} $u_2<0.1$, \textbf{then} $h_{ij} \sim N(0,30)$ \label{algo_pca_instance_generation_line_7}
\STATE \quad \quad \quad    \textbf{else} $h_{ij} \sim N(0,1)$ \label{algo_pca_instance_generation_line_8}
\STATE \quad \quad   \textbf{end for} \label{algo_pca_instance_generation_line_9}
\STATE \quad  \textbf{else} $h_{ij} \sim N(0,1)$ \label{algo_pca_instance_generation_line_10}
\STATE \textbf{end for} \label{algo_pca_instance_generation_line_11}
\STATE Obtain $P = U \Sigma V^{\top}$, the SVD of $P$
 \label{algo_pca_instance_generation_line_12}
\STATE Construct $A = (U[,1:q] + H) \Sigma[1:q,1:q] V^{\top}[,1:q]$ with $H = [h_{ij}] \in \mathbb{R}^{n \times q}$ generated in Steps \ref{algo_pca_instance_generation_line_2} - \ref{algo_pca_instance_generation_line_10} \label{algo_pca_instance_generation_line_13}
\STATE Adjust $A$ to have 0 mean for each column \label{algo_pca_instance_generation_line_14}
\end{algorithmic}
\end{algorithm}

 In Step \ref{algo_pca_instance_generation_line_1}, we first generate random matrix $P$, where each $p_{ij}$ is from the uniform distribution between -100 and 100.
Next in Steps \ref{algo_pca_instance_generation_line_2} - \ref{algo_pca_instance_generation_line_10}, we generate random perturbation matrix $H$ with approximately $r$ percent of rows having extremely large perturbations, where each row has approximately 10\% extreme value entries. After SVD of $P U \Sigma V^{\top}$ in Step \ref{algo_pca_instance_generation_line_12}, data matrix $A$ is generated, where $U[,1:q]$ is the submatrix of $U$ with the first $q$ columns, $\Sigma[1:q,1:q]$ is the submatrix of $\Sigma$ with the first $q$ columns and $q$ rows, and $V^{\top}[,1:q]$ is the submatrix of $V^{\top}$ with the first $q$ columns. The final data matrix $A$ is generated in Step \ref{algo_pca_instance_generation_line_14} after adjusting it to have 0 column means.

\subsubsection{UCI instances}

We also consider classification datasets from the UCI Machine Learning Repository \cite{Bache+Lichman:2013} and adjust them to create PCA instances. Based on the assumption that observations in the same class of a classification data set have similar attribute distributions, we consider each class of the classification datasets. For each dataset, we partition the observations based on labels (classes). When there exist many labels, we select the top two labels with respect to the number of observations among all labels. For each partitioned data, labels and attributes with zero standard deviation (hence, meaningless) are removed and the matrix is standardized to have zero mean and unit standard deviation for each attribute. 

In Table \ref{tab:uca_data}, we list the PCA instances we used and the corresponding original dataset from \cite{Bache+Lichman:2013}. In the first column, abbreviate names of the original data sets are presented. The full names of the data sets are Breast Cancer Wisconsin, Indian Liver Patient Dataset, Cardiotocography, Ionosphere, Connectionist Bench (Sonar), Landsat Satellite, Spambase, Magic Gamma Telescope, Page Blocks Classification, and Pen-Based Recognition of Handwritten Digits. Each PCA instance is classified as small or large based on $m$ and $n$. If $mn \leq 15,000$, the instance is classified as small, otherwise, the instance is classified as large. In the last column in Table \ref{tab:uca_data}, the small and large instances are indicated by $S$ and $L$, respectively. For the large instances, only \textsf{Kwak} and \textsf{Nie} are compared with the proposed algorithms, due to scalability issues of the other benchmark algorithms.

\begin{table}[htbp]
      \caption{PCA instances created based on the datasets from the UCI Machine Learning Repository \cite{Bache+Lichman:2013}}
  \label{tab:uca_data}%
\footnotesize
  \centering
    \begin{tabular}{|c|c|c||c|c|c|}
    \hline
    \multicolumn{3}{|c||}{Original dataset from UCI}   & \multicolumn{3}{c|}{PCA instance}  \\ \hline
    Name & $(m,n)$ & Num labels & Name  & $(m,n)$ & size \\ \hline
    cancer & (9,699) & 2     & cancer\_2 & (9,444) & S  \\
          &       &       & cancer\_4 & (9,239) & S  \\  \hline
    ilpd & (10,583) & 2     & ilpd\_1 & (10,416)& S  \\
          &       &       & ilpd\_2 & (10,167)& S \\  \hline
    cardio & (21,2126) & 10     & cardio\_1 & (19,384)& S \\
          &       &       & cardio\_2 & (19,579)& S  \\  \hline
    iono & (34,351) & 2     & iono\_b & (33,126)& S  \\
          &       &       & iono\_g & (32,225) & S \\  \hline
    sonar & (60,208) & 1     & sonar\_g & (60,111)& S  \\
          &       &       & sonar\_r & (60,97)& S \\  \hline
    landsat & (36,4435) & 7     & landsat\_1 & (36,1072) & L  \\
          &       &       & landsat\_3 & (36,961) & L  \\  \hline  
    spam & (57,4601) & 2     & spam\_0 & (57,2788)& L  \\
          &       &       & spam\_1 & (57,1813) & L \\  \hline
    magic & (10,19020) & 2     & magic\_g & (10,12332)& L \\
          &       &       & magic\_h & (10,6688)& L  \\  \hline
    blocks & (10,5473) & 5     & blocks\_1 & (10,4913)& L  \\  \hline
    hand  & (16,10992) & 10     & hand\_0 & (16,1142)& L  \\
  &       &       & hand\_1 & (16,1143)& L \\  \hline
    \end{tabular}%

\end{table}%

\subsection{Performance of awPCA for Synthetic Instances}
\label{section_exp_synthetic}

In Table \ref{tab:awpca_syn}, we present the result for \textsf{awPCA} for the synthetic instances. Although we created synthetic instances with varying $r$ (\% of outliers) values, we observed that the performances of the algorithms are very similar over different $r$ values for each $(m,n,p)$ triplet. Hence, in Table \ref{tab:awpca_syn}, we present the average value over all $r$. That is, each row of the table is the average of 20 instances for the corresponding $(m,n)$ pair given $p$. The first two columns are the instance size and number of PCs, the next five columns are $\Delta_{algo}$ for all algorithms, and the last five columns are the execution times in seconds. For each row, the lowest $\Delta_{algo}$ value among the five algorithms is boldfaced.

\begin{table}[htbp]
\setlength{\tabcolsep}{1pt}
\scriptsize
  \centering
  \caption{Performance of \textsf{awPCA} for synthetic instances} \label{tab:awpca_syn}
    \begin{tabular}{|rr|rrrrr|rrrrr|}
    \hline
    \multicolumn{2}{|c|}{Instance} & \multicolumn{5}{c|}{Gap from the best ($\Delta_{algo}$)} & \multicolumn{5}{c|}{Time (seconds)} \\ \hline
    ($m,n$) & $p$     & \textsf{awPCA} & \textsf{Ke}    & \textsf{Brooks} & \textsf{Kwak} & \textsf{Nie} & \textsf{awPCA} & \textsf{Ke}    & \textsf{Brooks} & \textsf{Kwak} & \textsf{Nie} \\ \hline
    (20, 100) & 8     & \textbf{1\%} & 6\%   & 2\%   & 12\%  & 18\%  & 0.0   & 0.7   & 1.3   & 0.0   & 0.0 \\
          & 9     & 4\%   & 22\%  & \textbf{3\%} & 16\%  & 26\%  & 0.0   & 0.5   & 1.3   & 0.0   & 0.0 \\
          & 10    & \textbf{0\%} & \textbf{0\%} & 1\%   & 7\%   & 7\%   & 0.0   & 0.4   & 1.3   & 0.0   & 0.0 \\
          & 11    & \textbf{0\%} & 69\%  & 2\%   & 7\%   & 12\%  & 0.0   & 0.4   & 1.3   & 0.0   & 0.0 \\
          & 12    & \textbf{0\%} & 70\%  & 2\%   & 7\%   & 16\%  & 0.0   & 0.4   & 1.2   & 0.0   & 0.0 \\ \hline
    (20, 300) & 8     & \textbf{2\%} & 8\%   & 3\%   & 10\%  & 12\%  & 0.0   & 4.9   & 9.5   & 0.0   & 0.1 \\
          & 9     & \textbf{3\%} & 10\%  & \textbf{3\%} & 13\%  & 19\%  & 0.0   & 2.8   & 9.6   & 0.0   & 0.1 \\
          & 10    & \textbf{0\%} & \textbf{0\%} & 1\%   & 3\%   & 3\%   & 0.0   & 1.2   & 9.6   & 0.0   & 0.1 \\
          & 11    & \textbf{0\%} & 9\%   & 1\%   & 3\%   & 6\%   & 0.0   & 1.2   & 9.6   & 0.0   & 0.1 \\
          & 12    & \textbf{0\%} & 7\%   & 1\%   & 3\%   & 8\%   & 0.0   & 1.2   & 9.6   & 0.0   & 0.1 \\ \hline
    (50, 100) & 8     & \textbf{1\%} & 3\%   & 2\%   & 13\%  & 16\%  & 0.0   & 3.8   & 19.2  & 0.0   & 0.0 \\
          & 9     & \textbf{1\%} & 7\%   & 3\%   & 15\%  & 21\%  & 0.0   & 2.9   & 19.1  & 0.0   & 0.0 \\
          & 10    & \textbf{0\%} & \textbf{0\%} & 3\%   & 7\%   & 7\%   & 0.0   & 2.9   & 19.3  & 0.0   & 0.0 \\
          & 11    & \textbf{0\%} & 100\% & 2\%   & 7\%   & 8\%   & 0.0   & 4.3   & 19.2  & 0.0   & 0.0 \\
          & 12    & \textbf{0\%} & 100\% & 3\%   & 7\%   & 10\%  & 0.0   & 4.8   & 19.2  & 0.0   & 0.0 \\ \hline
    (50, 300) & 8     & \textbf{1\%} & 3\%   & 2\%   & 9\%   & 13\%  & 0.1   & 27.0  & 227.4 & 0.0   & 0.1 \\
          & 9     & \textbf{2\%} & 5\%   & 3\%   & 10\%  & 16\%  & 0.1   & 22.0  & 226.5 & 0.0   & 0.2 \\
          & 10    & \textbf{0\%} & \textbf{0\%} & 1\%   & 3\%   & 3\%   & 0.0   & 18.7  & 226.7 & 0.0   & 0.2 \\
          & 11    & \textbf{0\%} & 86\%  & 1\%   & 3\%   & 4\%   & 0.0   & 23.1  & 228.1 & 0.0   & 0.2 \\
          & 12    & \textbf{0\%} & 98\%  & 1\%   & 3\%   & 5\%   & 0.0   & 27.2  & 226.0 & 0.0   & 0.2 \\ \hline
    \end{tabular}%
  \label{tab:addlabel}%
\end{table}%

Note that $\Delta_{\mbox{\scriptsize{\textsf{awPCA}}}}$ values are near zero for all instances. Further, $\Delta_{\mbox{\scriptsize{\textsf{awPCA}}}}$ has the lowest gaps (boldfaced numbers) for all instances among all algorithms except for one instance class. \textsf{Brooks} constantly gives the second best gaps while \textsf{Ke} gives 0\% gaps for $p = 10$, third best gaps for $p < 10$ and worst gaps for $p > 10$. \textsf{Nie} and \textsf{Kwak} generally give the similar result as they are designed to solve the same problem.

The execution times of the algorithms can be grouped into two groups: \textsf{awPCA}, \textsf{Kwak}, and \textsf{Nie} are in the faster group and \textsf{Ke} and \textsf{Brooks} are in the slower group. \textsf{Ke} and \textsf{Brooks} spend much more time on larger instances compared to the other three algorithms. Although it is not easy to compare, \textsf{Kwak} is the fastest among all algorithms, yet $\Delta_{\mbox{\scriptsize{\textsf{Kwak}}}}$ is not as low as  $\Delta_{\mbox{\scriptsize{\textsf{awPCA}}}}$. It is important to note that the difference in the execution time between \textsf{Kwak} and \textsf{awPCA} is negligible and \textsf{awPCA} is fastest among the algorithms designed to solve \textsf{P1} with the $L_1$ norm (i.e., \textsf{Ke}, \textsf{Brooks}, and \textsf{awPCA}).

\subsection{Performance of awPCA for UCI Instances}
\label{section_exp_uci}

\begin{figure*}
\centering
\includegraphics[width=1\textwidth]{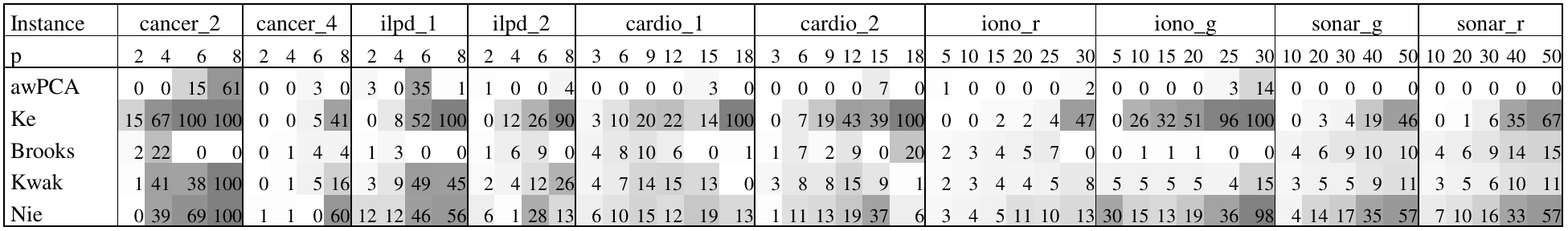}
\caption{Heat map of the gap (\%) from the best for small UCI instances}
\label{fig:wpca_small_uci_gap}
\end{figure*}

\begin{figure*}
\centering
\includegraphics[width=1\textwidth]{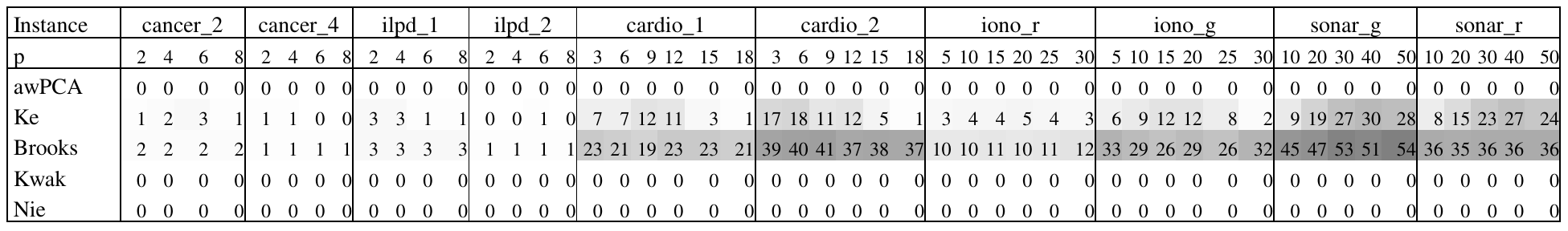}
\caption{Heat map of the execution times (seconds) for small UCI instances}
\label{fig:wpca_small_uci_time}
\end{figure*}

\begin{figure*}
\centering
\includegraphics[width=1\textwidth]{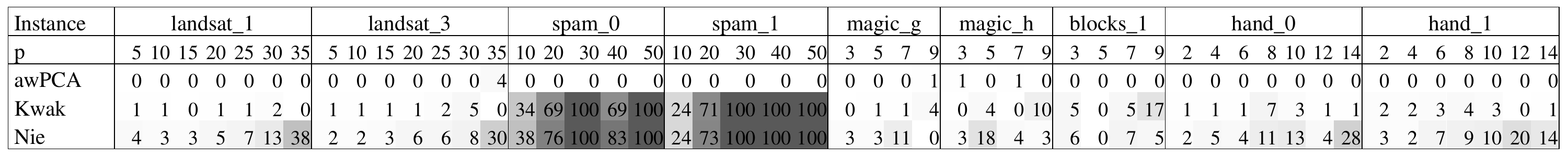}
\caption{Heat map of the gap (\%) from the best for large UCI instances}
\label{fig:wpca_large_uci_gap}
\end{figure*}

\begin{figure*}
\centering
\includegraphics[width=1\textwidth]{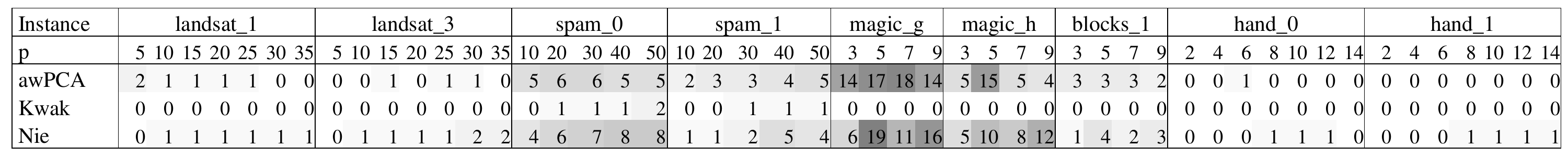}
\caption{Heat map of the execution times (seconds) for large UCI instances}
\label{fig:wpca_large_uci_time}
\end{figure*}

For each PCA instance in Table \ref{tab:uca_data}, we execute the algorithms with various $p$ values. The number of PCs $p$ covers the entire spectrum $0,\cdots,m$ in increments of 2,3,5, or 10 depending on $m$: \textit{cancer} and \textit{ilpd} with $p \in \{2,4,6,8\}$, \textit{cardio} with $p \in \{3,6,9,12,15\}$, \textit{iono} with $p \in \{5,10,15,20,25,30\}$, \textit{sonar} and \textit{spam} with $p \in \{10,20,30,40,50\}$, \textit{landsat} with $p \in \{5,10,\cdots,30,35\}$, \textit{magic} and \textit{block} with $p \in \{1,3,5,7,9\}$, and \textit{hand} with $p \in \{2,4,\cdots,12,14\}$.

For the small UCI instances, we present heat maps of $\Delta_{algo}$ and the execution times of all algorithms in Figures \ref{fig:wpca_small_uci_gap} and \ref{fig:wpca_small_uci_time}. In both figures, the numbers are $\Delta_{algo}$ in percentage or execution time in seconds, a white cell implies near-zero $\Delta_{algo}$ or near-zero execution time, and a dark gray cell implies the opposite. In Figure \ref{fig:wpca_small_uci_gap}, \textsf{awPCA} is consistently best with \textsf{Brooks} usually being the second best algorithm. The value of $\Delta_{\mbox{\scriptsize{\textsf{awPCA}}}}$ is zero except for a few cases. For such cases with $\Delta_{\mbox{\scriptsize{\textsf{awPCA}}}} > 0$, \textsf{Brooks} performs the best. The values of $\Delta_{\mbox{\scriptsize{\textsf{Ke}}}}$, $\Delta_{\mbox{\scriptsize{\textsf{Kwak}}}}$, and $\Delta_{\mbox{\scriptsize{\textsf{Nie}}}}$ tend to increase in $p$. In Figure \ref{fig:wpca_small_uci_time}, we observe the same trend from Section \ref{section_exp_synthetic}: \textsf{awPCA}, \textsf{Kwak}, and \textsf{Nie} are in the faster group and \textsf{Ke} and \textsf{Brooks} become slower as instance size increases.

For the large UCI instances, we only compare \textsf{awPCA} against \textsf{Kwak} and \textsf{Nie}, due to scalability issues of \textsf{Ke} and \textsf{Brooks}. Hence, $\Delta_{algo}$ here is the gap from the best of \textsf{awPCA}, \textsf{Kwak} and \textsf{Nie}. In Figures \ref{fig:wpca_large_uci_gap} and \ref{fig:wpca_large_uci_time}, we present heat maps of $\Delta_{algo}$ and the execution times of the three algorithms. Similar to Figures \ref{fig:wpca_small_uci_gap} and \ref{fig:wpca_small_uci_time}, a white cell implies a low value. In Figure \ref{fig:wpca_large_uci_gap}, \textsf{awPCA} is consistently best except for four cases and even for the four cases $\Delta_{\mbox{\scriptsize{\textsf{awPCA}}}}$ are very small. We observe that $\Delta_{\mbox{\scriptsize{\textsf{Kwak}}}}$ and $\Delta_{\mbox{\scriptsize{\textsf{Nie}}}}$ tend to increase in $p$, where $\Delta_{\mbox{\scriptsize{\textsf{Kwak}}}}$ is slightly smaller than $\Delta_{\mbox{\scriptsize{\textsf{Nie}}}}$ in general. In Figure \ref{fig:wpca_large_uci_time}, the execution time of \textsf{Kwak} is the fastest, and \textsf{awPCA} and \textsf{Nie} are of the same magnitude, although \textsf{awPCA} is slightly faster than \textsf{Nie}.

Based on the results for the UCI instances, we conclude that \textsf{awPCA} performs the best while the execution time of \textsf{awPCA} is of the same order or lower than the remaining algorithm.

\subsection{Comparison of wPCA and awPCA for UCI Instances}
\label{section_exp_compare_two}

In this section, we compare \textsf{wPCA} and \textsf{awPCA} for the UCI instances in terms of solution quality ($\Delta_{\mbox{\scriptsize{\textsf{wPCA}}}}$ and $\Delta_{\mbox{\scriptsize{\textsf{awPCA}}}}$) and execution time. In Table \ref{tab:compare}, we present the average performance of \textsf{wPCA} and \textsf{awPCA} for all $p$ values considered in Section \ref{section_exp_uci}. The fourth column is defined as \textit{diff} = $\Delta_{\mbox{\scriptsize{\textsf{wPCA}}}}-\Delta_{\mbox{\scriptsize{\textsf{awPCA}}}}$, where a negative \textit{diff} value implies that \textsf{wPCA} gives a better solution and near-zero \textit{diff} value implies that $\Delta_{\mbox{\scriptsize{\textsf{wPCA}}}}$ and $\Delta_{\mbox{\scriptsize{\textsf{awPCA}}}}$ are similar. The seventh column is defined as $ratio =$ execution time of \textsf{awPCA} / execution time of \textsf{wPCA}, where a less-than-one \textit{ratio} value implies that \textsf{awPCA} is faster than \textsf{wPCA}. In Table \ref{tab:compare}, we observe that $\Delta_{\mbox{\scriptsize{\textsf{wPCA}}}}$ and $\Delta_{\mbox{\scriptsize{\textsf{awPCA}}}}$ are very similar except for two instances (boldfaced values), while \textsf{awPCA} spends only 20\% of the time of \textsf{wPCA} on average. Note also that \textsf{awPCA} is not always inferior to \textsf{wPCA}. Although it is rare, \textsf{awPCA} gives a better solution than \textsf{wPCA} for instances $magic\_h$ and $cardio\_1$. In general, we found that $\Delta_{\mbox{\scriptsize{\textsf{awPCA}}}}$ is very similar or slightly larger than $\Delta_{\mbox{\scriptsize{\textsf{wPCA}}}}$, while \textsf{awPCA} is much faster. On the other hand, we can also ignore the time difference if execution times are within a few seconds. The UCI instances $spam\_0, spam\_1, magic\_g$, and $magic\_h$ with clear time difference between \textsf{wPCA} and \textsf{awPCA} have $mn > 50000$. Therefore, we recommend to use \textsf{wPCA} when the data size small and \textsf{awPCA} when the data size is very large.

\begin{table}[htbp]
\setlength{\tabcolsep}{4pt}
  \centering
  \small
  \caption{Comparison of wPCA and awPCA for UCI instances}
    \begin{tabular}{|l|rrr|rrr|}
    \hline
          & \multicolumn{3}{|c|}{Gap from the best ($\Delta_{algo}$)} & \multicolumn{3}{c|}{Time (seconds)} \\ \hline
    Instance & wPCA  & awPCA & diff  & wPCA  & awPCA & ratio \\ \hline
    cancer\_2 & 15.7\% & 19.0\% & \textbf{-3.3\%} & 1.7   & 0.1   & 0.1 \\
    cancer\_4 & 0.8\% & 0.8\% & 0.0\% & 0.4   & 0.0   & 0.1 \\
    ilpd\_1 & 1.0\% & 9.8\% & \textbf{-8.7\%} & 0.3   & 0.1   & 0.3 \\
    ilpd\_2 & 0.9\% & 1.2\% & -0.3\% & 0.5   & 0.0   & 0.0 \\
    cardio\_1 & 0.7\% & 0.6\% & 0.1\% & 0.3   & 0.1   & 0.3 \\
    cardio\_2 & 1.0\% & 1.2\% & -0.2\% & 1.1   & 0.1   & 0.1 \\
    iono\_r & 0.4\% & 0.4\% & 0.0\% & 0.4   & 0.0   & 0.1 \\
    iono\_g & 2.8\% & 2.8\% & 0.0\% & 0.2   & 0.0   & 0.2 \\
    sonar\_g & 0.1\% & 0.1\% & 0.0\% & 1.0   & 0.1   & 0.1 \\
    sonar\_r & 0.0\% & 0.0\% & 0.0\% & 0.9   & 0.1   & 0.1 \\ \hline
    landsat\_1 & 0.0\% & 0.0\% & 0.0\% & 3.2   & 0.6   & 0.2 \\
    landsat\_3 & 0.6\% & 0.6\% & 0.0\% & 5.7   & 0.4   & 0.1 \\
    spam\_0 & 0.0\% & 0.0\% & 0.0\% & 16.7  & 5.2   & 0.3 \\
    spam\_1 & 0.0\% & 0.0\% & 0.0\% & 18.9  & 3.2   & 0.2 \\
    magic\_g & 0.1\% & 0.2\% & -0.1\% & 47.7  & 15.6  & 0.3 \\
    magic\_h & 0.6\% & 0.5\% & 0.1\% & 13.8  & 7.1   & 0.5 \\
    blocks\_1 & 0.0\% & 0.0\% & 0.0\% & 1.8   & 0.3   & 0.2 \\
    hand\_0  & 0.0\% & 0.0\% & 0.0\% & 1.8   & 0.3   & 0.2 \\
    hand\_1  & 0.0\% & 0.0\% & 0.0\% & 2.2   & 0.3   & 0.1 \\ \hline
    average &       &       & -0.7\% &       &       & 0.2 \\ 
 \hline
    \end{tabular}%
  \label{tab:compare}%
\end{table}%

\section{Conclusions}

In this paper, we consider the $L_1$-PCA problem minimizing the L1 reconstruction errors and present iterative algorithms, \textsf{wPCA} and \textsf{awPCA}, where \textsf{awPCA} is an approximation version of \textsf{wPCA} developed to avoid computationally expensive operations of SVD. The core of the algorithms relies on an iteratively reweighted least squares scheme and the expressions in \eqref{eqn_property_f_star}. Although the optimality of $L_1$-PCA was not able to be shown and remains unknown, we show that the eigenvalues of \textsf{wPCA} and \textsf{awPCA} converge and that eigenvectors of \textsf{awPCA} converge. In the computational experiment, we observe that \textsf{awPCA} outperforms all of the benchmark algorithms while the execution times are competitive. Out of the four algorithms designed to minimize the L1 reconstruction errors (\textsf{Ke}, \textsf{Brooks}, \textsf{wPCA}, \textsf{awPCA}), we observe that \textsf{awPCA} is the fastest algorithm with near-best solution qualities.

\bibliographystyle{IEEEtran}
\bibliography{l1pca}
%

\end{document}